\documentclass[letterpaper]{article} % DO NOT CHANGE THIS
\usepackage[preprint]{aaai2027}  % DO NOT CHANGE THIS
\usepackage[hyphens]{url}  % DO NOT CHANGE THIS
\usepackage{graphicx} % DO NOT CHANGE THIS
\usepackage{natbib}  % DO NOT CHANGE THIS AND DO NOT ADD ANY OPTIONS TO IT
\usepackage{caption} % DO NOT CHANGE THIS AND DO NOT ADD ANY OPTIONS TO IT
\usepackage{algorithm}
\usepackage{algpseudocode}% For pseudocode in algorithms

\usepackage{amsmath}
\usepackage{longtable}

\usepackage{amssymb}      % Additional math symbols
\usepackage{amsthm}       % Theorem environments
\usepackage{enumitem}     % Customizable lists
\usepackage{times}         % Times New Roman font
\usepackage{booktabs}
\usepackage{multirow}
\usepackage{tabularx}
\usepackage{mathrsfs}

\usepackage{algorithm}

\newtheorem{proposition}{Proposition}

\usepackage{newfloat}
\usepackage{listings}
\DeclareCaptionStyle{ruled}{labelfont=normalfont,labelsep=colon,strut=off} % DO NOT CHANGE THIS
\floatstyle{ruled}
\newfloat{listing}{tb}{lst}{}
\floatname{listing}{Listing}

\usepackage{booktabs}

\title{Not All Ranks Are Equal: Budget-Aware LoRA Merging Across Tasks}
\author{
    Avinash Amballa,
    Yashas Malur Saidutta,
    Wenbo Li,
    Lazar Valkov, 
    Srinivas Chappidi
}
\affiliations{
    \textbf{Samsung Research America}
    \{a.amballa, ym.saidutta, wenbo.li1, lazar.valkov, vasu.c\}@samsung.com
}

\begin{document}

\maketitle

\begin{abstract}
Merging low-rank adapters (LoRAs) promises to eliminate the overhead of swapping task-specific weights at inference time. However, existing merging methods assume every layer needs the same rank budget. Further, some methods assume that rank budget needs to be split equally among the tasks too. We show this uniform-budget assumption is a major source of the performance gap between merged and per-task LoRAs. However, rank selection is an NP hard problem. To this end, we introduce \textbf{Net Utility}, a data free metric that first decomposes every task LoRA by its Singular Value Decomposition (SVD) and scores each of those singular directions by its task utility and its interference with other tasks’ directions. Next, we globally pool these scores to select singular directions with the highest values with a constraint on the total number of directions selected. The proposed Net Utility metric is applied on top of five different merging methods across three different merging spaces. The merging is done over two sets of tasks, vision and language tasks. Net utility based rank allocation outperforms its counterparts without that allocation. On average, over vision tasks it achieves +2.1\% improvement in performance, and +2.2\% improvement over the language tasks.
\end{abstract}

\section{Introduction}

% Introduction for: "Not All Ranks Are Equal: Budget-Aware LoRA Merging Across Tasks"
% AAAI format. AAAI's natbib setup makes \cite{} parenthetical -- (Author et al. 2024) --
% and \citet{} textual, so \cite is used here rather than \citep.

\label{sec:intro}
To overcome the inefficiencies of full-tuning of foundation models, parameter efficient fine-tuning has become a well adopted alternative. Low Rank Adaptation (LoRA)~\cite{hu2022lora} writes the weight update as a product of low-rank factors, $\Delta W = BA$, yielding a simpler deployment pattern, one frozen backbone plus a library of task-specific adapters. Serving $n$ tasks then means keeping $n$ adapters and swapping the active one based on the task. However this results in issues especially on on-device applications, where hot-swapping needs to be performed. LoRA Merging collapses them into a single low-rank adapter that serves every task with no routing or swapping.

\begin{figure}[t]
    \centering
\includegraphics[clip, width=0.45\textwidth]{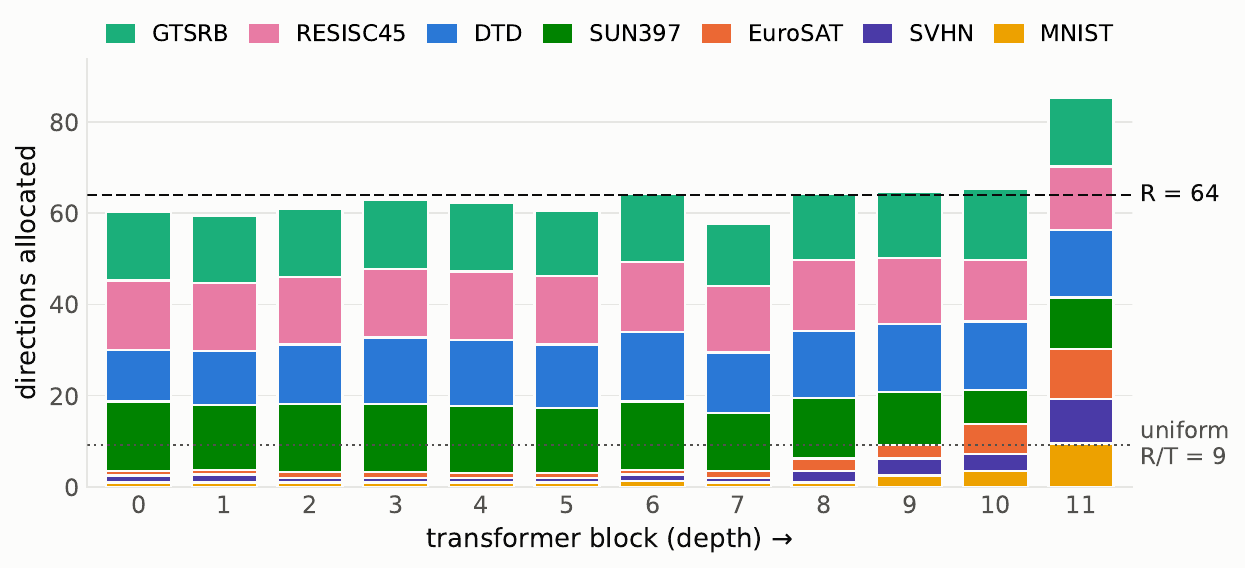}
    \caption{Ranks allocated to each task at different layers on a ViT-B/32 over seven vision tasks. We observe that different layers have different ranks indicating that certain layers are more important for LoRA updates than the others. We also observe that some tasks have less allocation at the lower layers and more at the later layers. The merging method is Task Arithmetic performed over the full space. Our non-uniform allocation improves accuracy by +3.4\%}
    \label{fig:hyp1-1}
\end{figure}

However, serving cost scales directly with adapter rank rather than adapter count:
decoding latency grows linearly in the ranks present in a batch, and mixed-rank
adapters force low-rank requests to pay the cost of the highest rank in the
batch~\cite{li2025toppings, jaiswal2025serving}. Deployment often fixes a rank,
independent of what training would have chosen~\cite{vulic2026lorasqueeze}. This is especially more crucial on on-device or edge applications. While many works have looked into the problem of merging \cite{yadav2023tiesmergingresolvinginterferencemerging, yu2024languagemodelssupermario, gargiulo2025tasksingularvectorsreducing, marczak2025taskleftbehindisotropic, panariello2025corespace}, but nearly all spread the rank budget uniformly across layers, i.e., every layer gets the same rank and some like \cite{gargiulo2025tasksingularvectorsreducing} explicitly assign every task an equal share. LoRA updates do not behave this way spectral energy can be distributed very unevenly across layers, and tasks can differ in which layers are important to them. So a uniform budget leads to under representation in some layers and over representation in others. In Figure~\ref{fig:hyp1-1} we show that relaxing this assumption leads to almost +3.4\% increase in average performance when tested across seven vision datasets with task arithmetic merging. This indicates that a rank allocation that takes layers and tasks into account during merging leads to superior merged adapters.

However, finding this optimal budget allocation is an NP hard problem. Hence, we introduce \textbf{Net Utility}, a criterion that helps us decide \emph{where} rank should be spent in a greedy manner. The SVD of each task's update yields candidate singular directions spanning every layer, and task; we score each by its contribution to reconstructing its own task's update, minus its conflict with the directions of other tasks, which gives a metric to score every singular direction. The score depends only on the adapter weights, is hence \textbf{data-free}. We pool all singular direction's net utility into a global pool and this allows us to design a method where the top-scoring directions are retained.

Our contributions are:
\begin{itemize}
    \item We identify uniform rank allocation, rather than the merging operation itself, as a major source of the gap between merged and per-task LoRAs.
    \item We propose Net Utility, a data-free score trading a singular direction's usefulness to its own task against its interference with other tasks, recasting merging as globally budget-constrained selection. 
    \item At matched budgets, on 7 vision and 6 language tasks, our Net Utility based allocation improves multiple prior merging methods over various merging spaces by an average of +2.1\% absolute, with some method-space combinations achieving as much as +3.8\% absolute.
\end{itemize}

\section{Related Work}

% Related Work for: "Not All Ranks Are Equal: Budget-Aware LoRA Merging Across Tasks"
% AAAI format: \cite{} is parenthetical, \citet{} is textual.
%
% Keys from your list (should be correct): yadav2023tiesmergingresolvinginterferencemerging,
% yu2024languagemodelssupermario, stoica2024modelmergingsvdtie,
% gargiulo2025tasksingularvectorsreducing, du2024parametercompetitionbalancingmodel,
% zheng2025decoupleorthogonalizedatafreeframework, qiao2025mergeforgetsinglelora,
% lee2026adarankadaptiverankpruning, he2026compressmergemultipleloras,
% marczak2025taskleftbehindisotropic, zhao2024mergingloraslikeplaying
%
% PLACEHOLDER keys you must verify: ilharco2023taskarithmetic, davari2024breadcrumbs, zhang2025lori, lee2026prime, shenaj2026kmerge, panariello2025corespace

% Related Work for: "Not All Ranks Are Equal: Budget-Aware LoRA Merging Across Tasks"
% AAAI format: \cite{} is parenthetical, \citet{} is textual.
%
% Keys from your list (should be correct): yadav2023tiesmergingresolvinginterferencemerging,
% yu2024languagemodelssupermario, stoica2024modelmergingsvdtie,
% gargiulo2025tasksingularvectorsreducing, du2024parametercompetitionbalancingmodel,
% zheng2025decoupleorthogonalizedatafreeframework, qiao2025mergeforgetsinglelora,
% lee2026adarankadaptiverankpruning, he2026compressmergemultipleloras,
% marczak2025taskleftbehindisotropic, zhao2024mergingloraslikeplaying
%
% PLACEHOLDER keys you must verify: ilharco2023taskarithmetic, davari2024breadcrumbs, zhang2025lori, lee2026prime, shenaj2026kmerge, panariello2025corespace

\label{sec:related}

\textbf{Task-vector merging:}
Task Arithmetic~\cite{ilharco2023taskarithmetic} sums the differences between fine-tuned and pretrained weights to obtain a multi-task model. Because independently trained vectors overlap, later work sparsifies before summing: by magnitude with a consensus sign (TIES~\cite{yadav2023tiesmergingresolvinginterferencemerging}), by random dropping with rescaling (DARE~\cite{yu2024languagemodelssupermario}), by discarding outliers as well as negligible weights (Breadcrumbs~\cite{davari2024breadcrumbs}), or by within-task importance weighted by its agreement across tasks (PCB-Merging~\cite{du2024parametercompetitionbalancingmodel}). We share PCB's logic to a degree: value to one's own task. However, PCB looks at inter-balancing, which is the benefit of a parameter to another task. Additionally, we score individual directions rather than parameters and use the scores to assign ranks rather than set a fixed sparsity ratio.

\noindent \textbf{Spectral and subspace merging:}
A second line merges in spectral bases: under a shared left singular basis (KnOTS~\cite{stoica2024modelmergingsvdtie}), by orthogonalizing truncated singular vectors across tasks (TSV-Merging~\cite{gargiulo2025tasksingularvectorsreducing}), by flattening the singular spectrum (Iso-Merging~\cite{marczak2025taskleftbehindisotropic}), or by splitting an ultra-low-rank principal component from a residual injected through its orthogonal complement (PRIME~\cite{lee2026prime}). These establish the singular basis as the right level at which to reason about interference, but fix the retained rank a priori and share it across layers---$d/n$ for TSV-Merging, roughly $1\%$ of the layer dimension for PRIME. Separately, ACE-Merging~\cite{xu2026acemergingdatafreemodelmerging} shows that each task's input covariance can be estimated from its fine-tuning update alone. Like TSV and Iso-C, it is a merge operator rather than a rank criterion and is orthogonal to the allocation question we study, i.e., our utility could be applied on top of it.

\noindent \textbf{LoRA Merging:}
Adapters require different treatment from full fine-tuning. DO-Merging~\cite{zheng2025decoupleorthogonalizedatafreeframework} attributes this to LoRA's much larger parameter-magnitude variance and merges magnitude and direction separately; LoRA-LEGO~\cite{zhao2024mergingloraslikeplaying} clusters rank-one components pooled across adapters into $k$ groups to assemble a rank-$k$ adapter; and Core Space~\cite{panariello2025corespace} merges within a common alignment basis, provably without information loss, which we use as one of the two spaces in which we allocate. Here too the retained rank is a single global constant.

\noindent \textbf{Rank selection and allocation.}
Closest to our work are methods that treat rank as something to be chosen. AdaRank~\cite{lee2026adarankadaptiverankpruning} observes, as we do, that dominant singular components are not necessarily the useful ones and that uniform truncation degrades performance, but prunes by learning a mask at test time through entropy minimization, requiring unlabelled calibration data and gradient descent. PRIME~\cite{lee2026prime} minimizes an interference--retention objective yet resolves it to one rank shared across all layers, and concurrently CtM~\cite{he2026compressmergemultipleloras} imposes the rank-$r$ bottleneck before rather than after merging. Net Utility is instead closed-form in the adapter weights, needing no data and no optimization loop, and selects no rank at all: because the scores are comparable across the entire pool, the rank profile falls out of a single top-$K$ over all modules, layers, and tasks. 

\noindent \textbf{Misc.:} Adjacent work targets merge-friendly training~\cite{zhang2025lori}, on-device continual merging under a budget on retained adapters~\cite{shenaj2026kmerge}, and sequential merging for continual learning~\cite{qiao2025mergeforgetsinglelora}, whereas we merge a fixed set of adapters offline under a total rank budget.

\section{Methodology: Net Utility and Rank Selection}
\label{sec:method}
% Preamble requirement: \newtheorem{proposition}{Proposition}
% This fragment owns \section{Method}; do not nest it under another Method section.

%\section{Method}
%\label{sec:method}

\begin{figure*}
    \centering
    \includegraphics[width=1\linewidth]{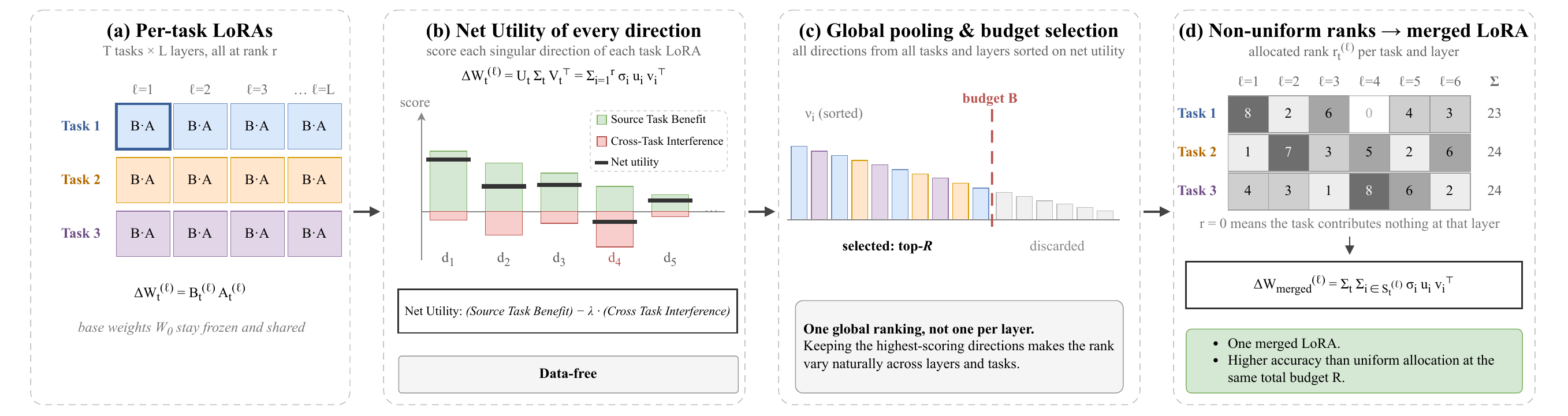}
    \caption{Net Utility based rank selection for Task Arithmetic Merging. (a): Obtain the per-task LoRAs. (b): Decompose all the LoRAs into their SVDs. Compute the Net Utility metric for each of the directions. (c): Pool all the directions and select the top-R directions. (d): Compose the merged LoRA using the selected directions. We naturally end up with variable number of directions being selected across tasks and layers.}
    \label{fig:overview}
\end{figure*}

In this section we introduce Net Utility based rank selection (see Fig.~\ref{fig:overview}). Net Utility decomposes each task adapter into rank-one singular components. At each adapted layer, it assigns every component a data-free score that balances source-task preservation against cross-task interference, then allocates a shared component budget across tasks before merging the retained components.

\subsection{Preservation Objective}
\label{sec:preservation-objective}

At merge time, only the $M$ task-specific LoRA updates $\{\Delta W_i\}_{i=1}^M$ and the base model are available; no task data, validation set, or forward pass is used.

Let $f_{\Delta W}$ denote the base model equipped with update $\Delta W$, and let $p_i$ be the input distribution of task $i$. We seek a merged update $\Delta W_{\mathrm m}$ whose model outputs match those of each individual adapter on its own inputs:
\begin{equation}
  \mathcal{L}_{\mathrm{pres}}
  =
  \sum_{i=1}^{M}
  \mathbb{E}_{x\sim p_i}
  \left[
    \left\lVert
      f_{\Delta W_i}(x)-f_{\Delta W_{\mathrm m}}(x)
    \right\rVert_2^2
  \right].
  \label{eq:preservation-objective}
\end{equation}

For $L$ adapted layers, let $h^l(x)$ be the input to layer $l$ under a shared reference model. Linearizing around this reference and assuming comparable sensitivity across layers results in the following surrogate (derived in Appendix~A.1):
\begin{equation}
  \begin{aligned}
    \widetilde{\mathcal{L}}_{\mathrm{pres}}
    &:=
    \sum_{l=1}^{L}\sum_{i=1}^{M}
    \mathbb{E}_{x\sim p_i}
    \left[
      \left\lVert
        \bigl(\Delta W_i^l-\Delta W_{\mathrm m}^l\bigr)h^l(x)
      \right\rVert_2^2
    \right] \\
    &=
    \sum_{l=1}^{L}\sum_{i=1}^{M}
    \left\lVert
      \Delta W_i^l-\Delta W_{\mathrm m}^l
    \right\rVert_{G_i^l}^2,
  \end{aligned}
  \label{eq:preservation-surrogate}
\end{equation}
where $G_i^l:=\mathbb{E}_{x\sim p_i}[h^l(x)h^l(x)^\top]$ and $\lVert X\rVert_G^2:=\operatorname{tr}(XGX^\top)$ is the associated squared weighted seminorm. The algorithm does not evaluate these activations, instead in the next section we approximate $G_i^l$ from the adapter weights.

\subsection{Data-Free Input Geometry}
\label{sec:data-free-geometry}

The matrix $G_i$ captures the second-order input statistics of task $i$, but it cannot be estimated without task data. Following \citet{xu2026acemergingdatafreemodelmerging}, we replace this unavailable geometry with a scaled input-side Gram proxy. For a compact SVD $\Delta W_i=U_i\Sigma_iV_i^\top$, the approximation is
\begin{equation}
  G_i
  \approx
  \kappa_i\Delta W_i^\top\Delta W_i
  =
  \kappa_iV_i\Sigma_i^2V_i^\top,
  \qquad \kappa_i>0.
  \label{eq:gram-geometry}
\end{equation}

The scale $\kappa_i$ need not be estimated because it cancels under the task-layer normalization introduced below. This Gram geometry satisfies $x^\top G_ix\approx\kappa_i\lVert\Delta W_i x\rVert_2^2$, so it gives greater weight to directions on which the adapter acts strongly. We use this geometry below and analyze one adapted layer, suppressing $l$ until the per-layer allocation rule.

\paragraph{Heuristic row-space geometry.}
As an alternative, we consider $\widehat G_i:=\kappa_iV_iV_i^\top$, which sets the singular-value weights to one on the adapter's row space and therefore weights all active right-singular directions equally. This construction is a heuristic rather than a consequence of the Gram approximation in Equation~\eqref{eq:gram-geometry}.

\subsection{Net-Utility Allocation}
\label{sec:net-utility-allocation}

Let $r_i=\operatorname{rank}(\Delta W_i)$. Its SVD expresses the update as $r_i$ rank-one singular components:
\begin{equation}
  D_{i,k}:=\sigma_{i,k}u_{i,k}v_{i,k}^{\top},
  \qquad
  \Delta W_i=\sum_{k=1}^{r_i}D_{i,k}.
  \label{eq:singular-components}
\end{equation}

The allocation retains or discards each component without rescaling it. With a binary selection variable $z_{i,k}$, the retained task update and additive merge are
\begin{equation}
  \begin{aligned}
    K_i&:=\sum_{k=1}^{r_i}z_{i,k}D_{i,k},
    &
    \Delta W_{\mathrm m}&:=\sum_{i=1}^{M}K_i,\\
    z_{i,k}&\in\{0,1\}.
  \end{aligned}
  \label{eq:additive-selection}
\end{equation}

At each layer, the number of selected components upper-bounds the rank of the additive merged update; its realized rank may be lower.

Retaining a component reduces the missing source-task update, but it may also introduce error on other tasks. The contribution of task $i$ to the layer-wise objective can be written as
\begin{equation}
  \begin{aligned}
    \ell_i
    &:=
    \left\lVert\Delta W_i-\Delta W_{\mathrm m}\right\rVert_{G_i}^2\\
    &=
    \left\lVert
      \bigl(\Delta W_i-K_i\bigr)-\sum_{j\ne i}K_j
    \right\rVert_{G_i}^2.
  \end{aligned}
  \label{eq:task-loss}
\end{equation}

The first term inside the seminorm is the missing source-task update, while the second collects components retained from other tasks. Squaring their difference introduces a signed cross-term, so the effect of any component depends on the complete selection.

\paragraph{Relative task-layer weighting.}
The absolute loss $\ell_i$ can be dominated by high-energy adapters \citep{marczak2025taskleftbehindisotropic}. We therefore use the relative loss $\widehat{\ell}_i:=\ell_i/\lVert\Delta W_i\rVert_{G_i}^2$ for every task-layer update with nonzero energy. Restoring the layer index, we minimize $\sum_l\sum_i\widehat{\ell}_i^l$. This normalization removes task-layer scale from the eventual utilities.

\paragraph{Separable upper bound.}
Equation~\eqref{eq:task-loss} couples the selected components, so they cannot be scored independently. Proposition~\ref{prop:separable-upper-bound} in Appendix~A.2 shows that, for any $\lambda>0$,
\begin{equation}
  \widehat{\ell}_i
  \le
  \left(1+\frac{M-1}{\lambda}\right)
  \frac{
    \left\lVert\Delta W_i-K_i\right\rVert_{G_i}^2
    +
    \lambda\sum_{j\ne i}\left\lVert K_j\right\rVert_{G_i}^2
  }{
    \left\lVert\Delta W_i\right\rVert_{G_i}^2
  }.
  \label{eq:separable-upper-bound}
\end{equation}

The first numerator term penalizes source-task components that are omitted, while the second penalizes components retained from other tasks. The bound removes the signed cross-term and separates the interference contributed by different tasks. The coefficient $\lambda$ controls the interference penalty. For fixed $\lambda$, the leading factor is positive and allocation-independent, so it can be omitted from the allocation objective.

For the Gram geometry, define the task scale $w_i:=\sum_{n=1}^{r_i}\sigma_{i,n}^4$, so that $\lVert\Delta W_i\rVert_{G_i}^2=\kappa_iw_i$. Proposition~\ref{prop:net-utility-decomposition} in Appendix~A.3 rewrites the resulting objective as a constant $M$ minus the total utility of the selected components. The utility of component $D_{j,k}$ is
\begin{equation}
    g_{j,k}
    :=
    \underbrace{
      \frac{\sigma_{j,k}^4}{w_j}
    }_{\text{source-task benefit $\tilde{\pi}$}}-
    \lambda
      \sigma_{j,k}^2
      \underbrace{\sum_{i\ne j}\sum_{n=1}^{r_i}
      \frac{\sigma_{i,n}^2}{w_i}
      \left(v_{j,k}^{\top}v_{i,n}\right)^2
    }_{\text{cross-task interference $ \tilde{C}$}}.
  \label{eq:net-utility}
\end{equation}

The first term is the normalized source-task benefit, the second is cross-task interference and vanishes for directions orthogonal to every other task's row space. The resulting utility is dimensionless, allowing scores to be compared across tasks within the layer.

\paragraph{Global budget.}
% Restoring the layer index, we pool the components from all tasks within each adapted layer $l$ under a budget $R$:

Restoring the layer index, we pool the components from all tasks and all adapted layers under a total budget $R_G = L\cdot R$, i.e.\ an average of $R$ per layer, which allows budget to move between layers $L$

\begin{equation}
  \begin{aligned}
    \underset{z^1,...,z^L}{\operatorname{maximize}}
    &\quad
    \sum_{l=1}^{L}\sum_{j=1}^{M}\sum_{k=1}^{r_j^l}
    z_{j,k}^l g_{j,k}^l\\
    \text{subject to}
    &\quad
    z_{j,k}^l\in\{0,1\},\\
    &
    \sum_{l=1}^{L}\sum_{j=1}^{M}\sum_{k=1}^{r_j^l}
    z_{j,k}^l\le R_G
  \end{aligned}
  \label{eq:per-layer-allocation}
\end{equation}

Because every component has a fixed score, Equation~\eqref{eq:per-layer-allocation} is maximized by retaining the largest positive scores across tasks, up to $R_G$. Capacity remains unused if fewer scores are positive. This selection is exact for the separable surrogate, not for the original model-level objective.

The derivation applies to the additive merge in Equation~\eqref{eq:additive-selection}. Applying the selected components through a nonlinear merge operator is an empirical extension without the same surrogate guarantee.

\subsection{Algorithm}
\label{sec:net-utility-algorithm}

Algorithm~\ref{alg:net-utility} summarizes the method for the Gram geometry $G_i$, using the global allocation in Equation~\eqref{eq:per-layer-allocation}. The corresponding variant for $\widehat G_i$ follows analogously by using the utility derived in Appendix~A.4.

\begin{algorithm}[t]
  \caption{Net-utility rank allocation for LoRA merging}
  \label{alg:net-utility}
  \begin{algorithmic}[1]
    \Require Task updates $\{\Delta W_i^l\}$, interference coefficient $\lambda$
    \Require Global component budget $R_G = L \cdot R$, merge operator $\operatorname{Merge}$
    \For{each adapted layer $l$}
      \State Compute the SVD of every $\Delta W_i^l$ using QR decomposition
      \State Compute $w_i^l=\sum_n(\sigma_{i,n}^l)^4$
      \State Form $V_j^{l\top}V_i^l$ and compute utilities using Equation~\eqref{eq:net-utility}
    \EndFor
    \State Select the top positive $R_G$ scores across all tasks and layers.
    \State Reconstruct each $K_i^l$ from its selected components
    \State \Return $\operatorname{Merge}(\{K_i^l\}_{i,l})$
  \end{algorithmic}
\end{algorithm}

The SVDs can be computed directly from the LoRA factors, and neither the dense updates nor $G_i^l$ need to be materialized. Appendix~A.5 gives the computational complexity.

\begin{table*}[t]
\centering
\small
\setlength{\tabcolsep}{4pt}
\begin{tabular}{lcl ccc ccc ccc}
\toprule
& & & \multicolumn{3}{c}{$R=64$} & \multicolumn{3}{c}{$R=32$} & \multicolumn{3}{c}{$R=16$} \\
\cmidrule(lr){4-6}\cmidrule(lr){7-9}\cmidrule(lr){10-12}
Space & Iso & Method & Unif. & Ours & $\Delta$ & Unif. & Ours & $\Delta$ & Unif. & Ours & $\Delta$ \\
\midrule
\multirow{1}{*}{Full} & \multirow{1}{*}{--}
  & TA              & 63.5 & 66.9 & $\mathbf{+3.4}$ & 63.1 & 65.9 & $\mathbf{+2.7}$ & 62.5 & 64.9 & $\mathbf{+2.3}$ \\
\midrule
\multirow{3}{*}{Full} & \multirow{3}{*}{--}
 & TSV             & 66.4 & 66.9 & $\mathbf{+0.4}$ & 65.4 & 65.7 & $\mathbf{+0.3}$ & 63.9 & 64.7 & $\mathbf{+0.8}$ \\
& & DARE            & 63.5 & 67.3 & $\mathbf{+3.8}$ & 63.1 & 66.1 & $\mathbf{+3.0}$ & 62.6 & 65.8 & $\mathbf{+3.2}$ \\
& & TIES            & 62.5 & 65.0 & $\mathbf{+2.5}$ & 62.3 & 64.8 & $\mathbf{+2.4}$ & 61.8 & 64.8 & $\mathbf{+3.0}$ \\
% & & CART            & 64.9 & 63.8 & $-1.1$ & 64.6 & 63.8 & $-0.9$ & 64.1 & 63.8 & $-0.3$ \\
\midrule
\multirow{3}{*}{Core} & \multirow{3}{*}{--}
 & TSV             & 66.4 & 66.8 & $\mathbf{+0.4}$ & 65.4 & 65.7 & $\mathbf{+0.3}$ & 63.9 & 64.6 & $\mathbf{+0.7}$ \\
& & DARE            & 63.8 & 67.0 & $\mathbf{+3.2}$ & 63.4 & 66.4 & $\mathbf{+3.0}$ & 62.8 & 65.7 & $\mathbf{+2.9}$ \\
& & TIES            & 62.8 & 65.1 & $\mathbf{+2.3}$ & 62.4 & 64.7 & $\mathbf{+2.3}$ & 62.2 & 64.2 & $\mathbf{+2.0}$ \\
% & & CART            & 64.8 & 63.7 & $-1.1$ & 64.6 & 63.8 & $-0.8$ & 64.1 & 63.8 & $-0.3$ \\
% & & CART ($\lambda{=}0$) & 64.8 & 64.7 & $-0.1$ & 64.6 & 64.3 & $-0.3$ & 64.1 & 63.9 & $-0.2$ \\
% & & CART (reord.)   & 64.9 & 64.8 & $-0.0$ & 64.6 & 64.0 & $-0.6$ & 64.1 & 63.8 & $-0.2$ \\
\midrule
\multirow{3}{*}{KnOTS} & \multirow{3}{*}{--}
  & TSV             & 66.4 & 66.9 & $\mathbf{+0.5}$ & 65.4 & 65.8 & $\mathbf{+0.4}$ & 63.9 & 64.6 & $\mathbf{+0.8}$ \\
& & DARE            & 63.6 & 66.9 & $\mathbf{+3.3}$ & 63.2 & 66.3 & $\mathbf{+3.0}$ & 62.6 & 64.9 & $\mathbf{+2.3}$ \\
& & TIES            & 62.9 & 65.3 & $\mathbf{+2.4}$ & 62.6 & 65.0 & $\mathbf{+2.4}$ & 62.2 & 64.3 & $\mathbf{+2.1}$ \\
% & & CART            & 64.9 & 63.7 & $-1.2$ & 64.6 & 63.8 & $-0.8$ & 64.1 & 63.8 & $-0.3$ \\
\bottomrule
\end{tabular}
\caption{Vision tasks, ViT-B/32, 7 tasks. Normalized accuracy averaged over all 7 tasks.}
\label{tab:main-v}
\end{table*}

\section{Experiments}

\begin{table*}[t]
\centering
\small
\setlength{\tabcolsep}{4pt}
\begin{tabular}{lcl ccc ccc ccc}
\toprule
& & & \multicolumn{3}{c}{$R=64$} & \multicolumn{3}{c}{$R=32$} & \multicolumn{3}{c}{$R=16$} \\
\cmidrule(lr){4-6}\cmidrule(lr){7-9}\cmidrule(lr){10-12}
Space & Iso & Method & Unif. & Ours & $\Delta$ & Unif. & Ours & $\Delta$ & Unif. & Ours & $\Delta$ \\
\midrule
\multirow{1}{*}{-} & \multirow{1}{*}{--}
  & TA    & 72.8 & 75.1 & $\mathbf{+2.3}$ & 72.8 & 75.0 & $\mathbf{+2.3}$ & 72.6 & 74.8 & $\mathbf{+2.2}$ \\
\midrule
\multirow{3}{*}{Full} & \multirow{3}{*}{--}
  & TSV   & 78.3 & 81.4 & $\mathbf{+3.1}$ & 77.5 & 80.6 & $\mathbf{+3.1}$ & 76.8 & 80.3 & $\mathbf{+3.5}$ \\
& & DARE  & 72.3 & 73.7 & $\mathbf{+1.4}$ & 72.5 & 73.3 & $\mathbf{+0.9}$ & 72.0 & 73.4 & $\mathbf{+1.3}$ \\
& & TIES & 71.9 & 73.6 & $\mathbf{+1.7}$ & 71.3 & 73.6 & $\mathbf{+2.3}$ & 70.2 & 73.5 & $\mathbf{+3.4}$ \\
\midrule
\multirow{3}{*}{Core} & \multirow{3}{*}{--}
  & TSV   & 78.3 & 81.4 & $\mathbf{+3.1}$ & 77.5 & 80.6 & $\mathbf{+3.1}$ & 76.8 & 80.3 & $\mathbf{+3.5}$ \\
& & DARE & 68.2 & 68.1 & $-0.1$ & 68.2 & 68.1 & $-0.1$ & 66.0 & 67.0 & $\mathbf{+1.1}$ \\
& & TIES & 74.4 & 77.5 & $\mathbf{+3.1}$ & 74.3 & 77.5 & $\mathbf{+3.2}$ & 74.6 & 77.5 & $\mathbf{+2.9}$ \\
\midrule
\multirow{3}{*}{KnOTS} & \multirow{3}{*}{--}
  & TSV   & 78.3 & 81.4 & $\mathbf{+3.1}$ & 77.5 & 80.6 & $\mathbf{+3.1}$ & 76.8 & 80.3 & $\mathbf{+3.5}$ \\
& & DARE  & 71.1 & 71.1 & $\mathbf{+0.1}$ & 70.9 & 70.5 & $-0.4$          & 72.8 & 72.5 & $-0.3$ \\
& & TIES & 73.3 & 75.7 & $\mathbf{+2.4}$ & 72.9 & 75.7 & $\mathbf{+2.8}$ & 72.4 & 75.7 & $\mathbf{+3.2}$ \\
\bottomrule
\end{tabular}
\caption{Language (NLI) tasks, Qwen3-4B, 6 tasks. Normalized accuracy averaged over
6 tasks.}
\label{tab:main-l}
\end{table*}

\begin{table*}[t]
\centering
\small
\setlength{\tabcolsep}{5pt}
\begin{tabular}{@{}ll ccc ccc@{}}
\toprule
& & \multicolumn{3}{c}{Vision (ViT-B/32, 7 tasks)}
& \multicolumn{3}{c}{Language (Qwen3-4B, 6 tasks)} \\
\cmidrule(lr){3-5}\cmidrule(lr){6-8}
Space & Method & Unif. & Ours & $\Delta$ & Unif. & Ours & $\Delta$ \\
\midrule
Full & Iso-C + TA  & 66.0 & 67.3 & $\mathbf{+1.3}$ & 81.6 & 82.6 & $\mathbf{+1.0}$ \\
Full & Iso-C + TSV & 66.0 & 67.2 & $\mathbf{+1.2}$ & 81.8 & 82.3 & $\mathbf{+0.5}$ \\
\bottomrule
\end{tabular}
\caption{Isotropization (Iso-C) at $R{=}16$, full space for both Vision and Language tasks.}
\label{tab:iso}
\end{table*}

\paragraph{Experimental details} 
We use the experimental setup of KnOTS \cite{stoica2024modelmergingsvdtie} and use the LoRA checkpoints for vision tasks provided\footnote{https://huggingface.co/collections/hoffman-lab/knots-model-merging-with-svd}. For the language tasks, we train LoRAs corresponding to the datasets below on the QWEN3-4B model (training details in Appendix~B). All LoRAs have rank 16 applied on the matrices of key projection, query projection, value projection, and output projection, across all attention layers. Following prior work like \cite{stoica2024modelmergingsvdtie, panariello2025corespace}, we report normalized accuracy as a ratio of the accuracy of the merged LoRA on a given task to the accuracy of the original LoRA on this task. We implement on top of the codebase from \cite{panariello2025corespace}\footnote{https://github.com/apanariello4/core-space-merging}. All experiments are run on NVIDIA H100 GPUs.

\begin{enumerate}
    \item \textbf{Vision tasks}: For the vision experiments, we use CLIP ViT-B/32 \cite{dosovitskiy2021imageworth16x16words} as vision encoders fine-tuned on a standard set of 7 tasks DTD, EuroSAT, GTSRB, MNIST, RESISC, SUN397, SVHN.
    \item \textbf{Language tasks}: For language experiments, we use Qwen 3-4B \cite{yang2025qwen3technicalreport} fine-tuned on 6 NLI tasks SNLI, MNLI, SICK, QNLI, RTE, SCITAIL. 
\end{enumerate}

\paragraph{Baselines}
We use multiple merging methods as baselines. 
\begin{enumerate}
    \item Task Arithmetic (TA) \cite{ilharco2023taskarithmetic} performs a scaled summation of each task matrix. 
    \item TIES \cite{yadav2023tiesmergingresolvinginterferencemerging} trims low-magnitude parameters and averages parameters with majority sign. 
    \item DARE \cite{yu2024languagemodelssupermario} preprocesses task vectors by randomly dropping a fraction of parameters and rescaling to the mean. 
    \item TSV \cite{gargiulo2025tasksingularvectorsreducing} concatenates low-rank approximations of task matrices and orthogonalizes them across tasks. 
    \item Iso-C \cite{marczak2025taskleftbehindisotropic} flattens the spectrum of singular values for a model merged with task arithmetic. 
\end{enumerate}
We apply Net Utility on all the methods discussed above. However, for TIES and DARE, the budgeting is done locally, i.e., each layer gets an equal budget. This is because TIES and DARE apply sample specific operations. DARE drops elements randomly for each sample and TIES performs the sign-resolution at merge time and depends on the searched pruning rate or $\mathrm{top}K$.

We also apply our net utility allocation on the methods discuss above in three merge spaces: the full weight space, core space \cite{panariello2025corespace}, and KnOTS space \cite{stoica2024modelmergingsvdtie}. Each method is run at three budgets $R \in \{64, 32, 16\}$, where $R$ is the number of singular directions retained per layer across all tasks. Full capacity is $T\!\cdot\!r$ where $T$ is the number of tasks, i.e.\ $112$ for vision and $96$ for language, so all three budgets we evaluate are less than the full budget. 
% The baseline \emph{Unif.} splits the budget evenly across tasks and keeps each task's top singular directions, which is what TSV and KnOTS do. 

\paragraph{Hyperparameters}
In \eqref{eq:gram-geometry}, we have a choice of $\widehat G_i:=\kappa_iV_iV_i^\top$ and $\widehat G_i:=\kappa_iV_i \Sigma^2_i V_i^\top$. We can write this as $\widehat G_i:=\kappa_iV_i \Sigma^{(2\alpha)}_i V_i^\top$, where $\alpha \in \{0,1\}$. The exponent $\alpha$ amplifies the magnitude of the task update. When the task update magnitudes are roughly similar, this is not an issue. But when the task energies are very different it can end up focusing on only the tasks with the highest. To test this, we use variance of logarithms, a measure for heterogeneity \cite{foster1999lorenz}.
\begin{equation}
h \;:=\; \operatorname{Var}_i\!\big[\log\lVert\Delta W_i\rVert_F^{2}\big],
\label{eq:h}
\end{equation}
We find $h=1.83$ across the seven vision adapters and $h=0.63$ across the six NLI adapters. The vision tasks are more heterogenous and we set $\alpha=0$ to ensure that tasks with large magnitudes do not dominate the merged LoRA.

We use a data free approach to assigning $\lambda=\mathbf{M}(\tilde{\pi})/\mathbf{M}(\sigma^{2}\tilde{C})$ in \eqref{eq:net-utility} and $\mathbf{M}$ is the median.

\begin{table*}[t]
  \centering
  \begin{tabular}{@{}lcc c cccc@{}}
    \toprule
    & \multicolumn{3}{c}{spend the whole budget}
    & \multicolumn{4}{c}{admit only $g_{j,l}>0$} \\
    \cmidrule(lr){2-4}\cmidrule(lr){5-8}
    $R$ & Uniform & Ours & $\Delta$ & Ours & used (avg) & unspent & $\Delta$ \\
    \midrule
    $112$ (full budget) & $63.75$ & $63.75$ & $+0.00$
                 & $\mathbf{66.94}$ & $61.5$ & $\mathbf{45.0\%}$ & $\mathbf{+3.19}$ \\
    $64$         & $63.54$ & $\mathbf{66.94}$  & $\mathbf{+3.40}$
                 & $66.94$          & $61.5$ & $3.9\%$           & $+3.40$ \\
    $32$         & $63.15$ & $\mathbf{65.87}$  & $\mathbf{+2.72}$
                 & $65.87$          & $32.0$ & $0\%$             & $+2.72$ \\
    $16$         & $62.52$ & $\mathbf{64.86}$  & $\mathbf{+2.34}$
                 & $64.86$          & $16.0$ & $0\%$             & $+2.34$ \\
    \bottomrule
  \end{tabular}
  \caption{The budget need not be spent. Net-utility allocation on task  arithmetic on vision tasks. \emph{used} is directions retained per module. At full rank $R = 112$ only $55\%$ carry positive net utility, discarding the rest raises accuracy by $3.19$ while leaving over half the budget unspent. At lower ranks, we observe that entire budget is spent.}
  \label{tab:tv-global}
\end{table*}

\begin{figure*}[t]
\centering
\includegraphics[width=0.8\textwidth]{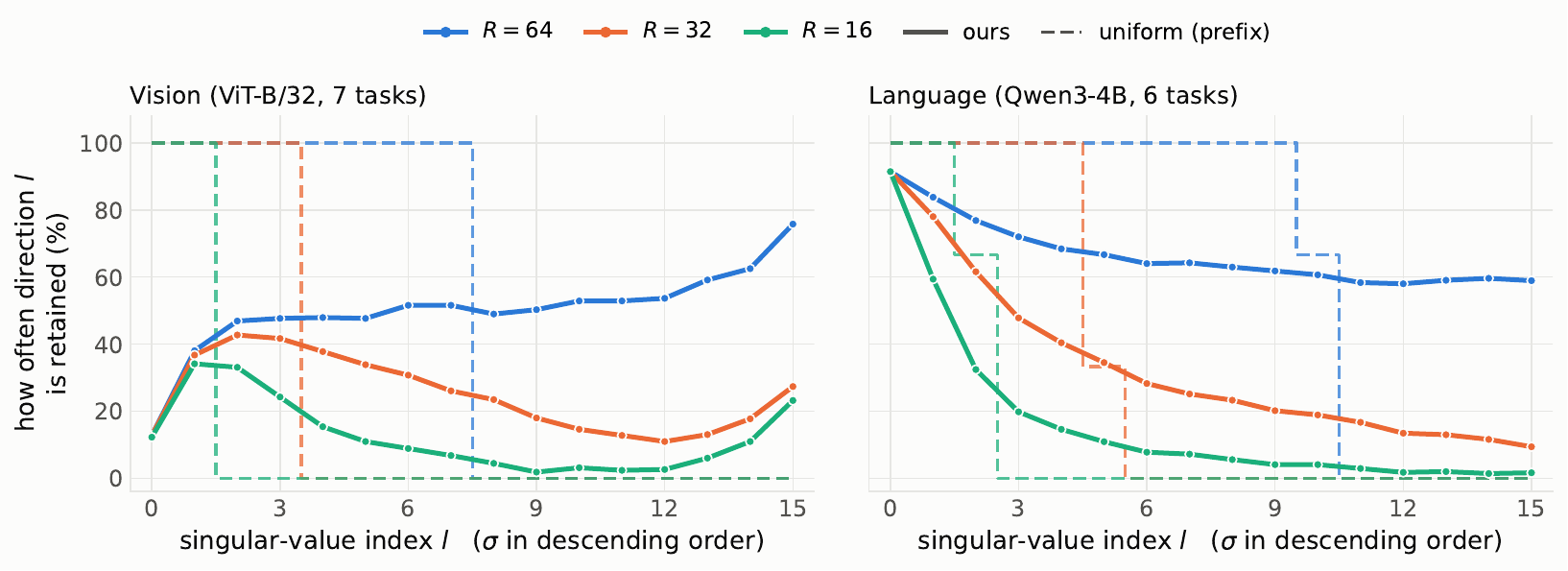}
\caption{Fraction of sets that retain the $l$-th singular direction. A prefix rule keeps every direction above a cut and none below it, i.e.\ the dashed step. On the other hand the net utility produces nothing of the sort.}
\label{fig:prefix}
\end{figure*}

\section{Results}

\paragraph{Vision.}
From Table~\ref{tab:main-v}, we observe that our rank allocation outperforms uniform
allocation across all three spaces and all four
merging methods. We achieve a maximum gain of $\mathbf{+3.8}$ over DARE in the full space with a budget of rank $64$. The gains are largest for DARE that improves by $+2.3$ to $+3.8$ and TA by $+2.3$ to $+3.4$, while TSV
improves to a lesser extent by $+0.3$ to $+0.8$ across the rank budgets. This could be because TSV whitens and re-ranks the spectrum, hence uniform allocation is a strong baseline there.

\paragraph{Language.}
We observe similar trends in Table~\ref{tab:main-l}. Our allocation improves over uniform in most of the methods across all the spaces, with a maximum gain of $\mathbf{+3.5}$ on TSV at budget $16$.
Our method's gains over TSV range from $+3.1$ to $+3.5$ in all three spaces, TIES range from $+1.7$ to $+3.4$, and gains over TA range from $+2.2$ to $+2.3$. The exception is DARE in core and KnOTS space, where we are within $0.4$ of uniform and slightly below it at budgets $64$ and $32$. Unlike vision, TSV is the method that benefits most here. We also observe that the gain tends to grow as the budget shrinks i.e., TSV improves from $+3.1$ at $R{=}64$ to $+3.5$ at $R{=}16$, and TIES in full space from $+1.7$ to $+3.4$, which is the expected behaviour, since at a tight budget the choice of which directions to keep matters more. The reason why TSV excels here could be attributed to the relative homogeneity of the language tasks indicating that rank allocation is more important in that regime.

We also ran experiments with the Iso-C \cite{marczak2025taskleftbehindisotropic}, which
replaces the retained singular values by their mean. Table~\ref{tab:iso} reports
these at budget $16$ for both domains. Our allocation improves over uniform in all
four settings: $+1.3$ and $+1.2$ on vision, $+1.0$ and $+0.5$ on language.

Our net utility metrics although derived based task arithmetic, i.e.\ $\sum_i \Delta W_i(\mathcal{S}_i)$ of the retained parts transfer well to others. Note, TSV, DARE and TIES are not linear in the task updates i.e., they apply truncation, random masking, and trimming with sign election respectively so the decomposition is not exact for them. However, using the net utility derived for task arithmetic on these methods improve their performances as shown in the Tables \ref{tab:main-v} and \ref{tab:main-l}. This shows that net utility is transferable across different merge methods and different spaces.

\section{Analysis}

\subsection{The budget need not be spent}
Because $g_{j,l}$ can be negative, admitting a direction whose interference exceeds its relative self-gain is strictly worse than an empty slot. In this experiment, we show that our method drops singular directions whose $g_{j,l} < 0$ on task arithmetic merging on vision tasks at higher budgets. Classical rank allocation always use the entire $R$. In Table \ref{tab:tv-global} we show that at budget of 112, our algorithm retains only $61.5$ directions on average per layer, even though the budget allows 112 per layer. That is only 48$\%$ of the budget is spent, resulting in 3.19 $\%$ more accuracy than using all the budget. At budget of $64$, the number of ranks used is similar and again it does not use the entire budget. At lower budgets, it uses all available budget as the number of directions with $g_{j,l} > 0$ exceeds the number of available slots.

\subsection{The optimal set need not be a prefix}
\label{rem:prefix}
We observe that net-utility $g_{j,l}$ need not be monotone in $l$ i.e., although $\tilde{\pi}_{j,l}$ decreases with $l$, the interference $\sigma_{j,l}^2\tilde{C}_{j,l}$ is governed by
cross-task geometry, so a low-$\sigma$ direction orthogonal to the other tasks can outrank a dominant but heavily shared singular value. Hence the optimal set $\mathcal{S}_j^\star$ is in general a subset, and hence a prefix truncation used by all existing SVD-based merging methods is suboptimal. On vision, $88\%$ of the choosen sets are not prefixes at every budget. On language, $64.5\%$ of the sets are non-prefix at $R{=}64$. Figure~\ref{fig:prefix} shows that on vision tasks, the relation is close to inverted i.e., at $R{=}64$ the strongest direction of a task is retained only $12.8\%$ of the time while the weakest is retained $75.8\%$ of the time. On language the curve decreases with $l$ but still retains $59\%$ of the directions at $l=15$, where a prefix rule at the same budget retains none.

% \avinash{Add experimental evidence}

\begin{figure}[t]
  \centering
\includegraphics[clip, width=0.48\textwidth]{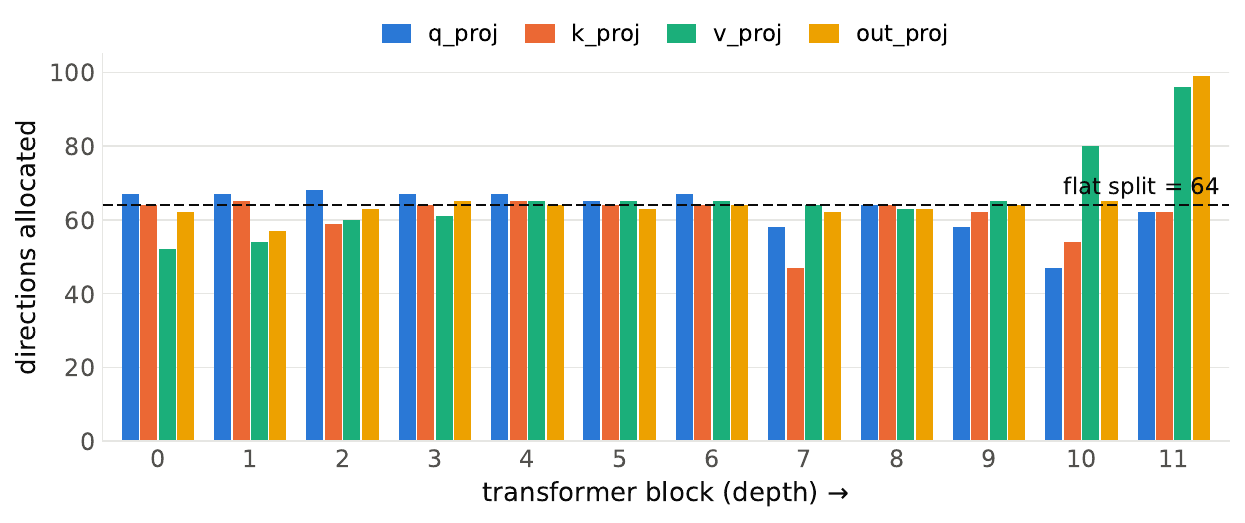}
  \caption{Total Budget allocation across all 7 vision tasks v/s layers. We observe that later layers require more budget than early layers. We also observe that modules v-proj and out-proj occupy more budget in the later layers. }
  \label{fig:Co3}
  \vspace{-5mm}
\end{figure}

\subsection{Not all Layers/Modules are equally important}

In Fig.~\ref{fig:Co3} we show the allocated ranks to each module \& layer combination for the ViT-B/32 model on the seven vision tasks with a budget of $64$. We find that different modules and layers deviate pretty heavily from the uniform rank. Specifically, in the later layers the value projection and output projection layers require more ranks. These modules also appear to be the most important matrices where the budget allocation plays a major role. The figure corresponds to Task Arithmetic used as the merging method over full space.

\subsection{Analyzing the net utility}
In the previous sections we show that the allocation induced by $g_{j,l}$ improves merging, but not what $g_{j,l}$ measures about a task. In this section we conduct an experiment to interpret what the net utility measures.

This requires a more controlled experiment and therefore we use
Hendrycks MATH \cite{hendrycks2021measuringmathematicalproblemsolving}, which splits a single MATH task into seven topics (algebra, counting and probability, geometry, intermediate algebra, number theory, prealgebra, precalculus) and we train one LoRA adapter per topic on Qwen3 with identical training setup (rank $32$), so the seven adapters differ only in the topic they were trained on.

We measure the \emph{relative} accuracy, i.e.\ the adapter's accuracy minus the base model's accuracy on the same topic, which controls for what the base model already knew. A higher relative accuracy means an easier
topic. We can also observe similar trends in the Figure~\ref{fig:analysis} where intermediate algebra has lower relative accuracy than prealgebra and algebra. Similarly precalculus is one of the harder subjects.

We also observe in Figure~\ref{fig:analysis} that the relative accuracy is inversely correlated with the mean net utility, with a correlation of $-0.78$. Since a higher relative accuracy means an easier
topic, this says that harder topics have \emph{higher} net utility. This is consistent with how $g_{j,l}$ is defined. Harder topics appear to require directions that the other topics do not use, while easier topics are largely covered by what the rest of the set already provides. Thus we interpret that net-utility could be associated to the task difficulty. 

\begin{figure}[t]
    \centering
    \includegraphics[width=0.4\textwidth]{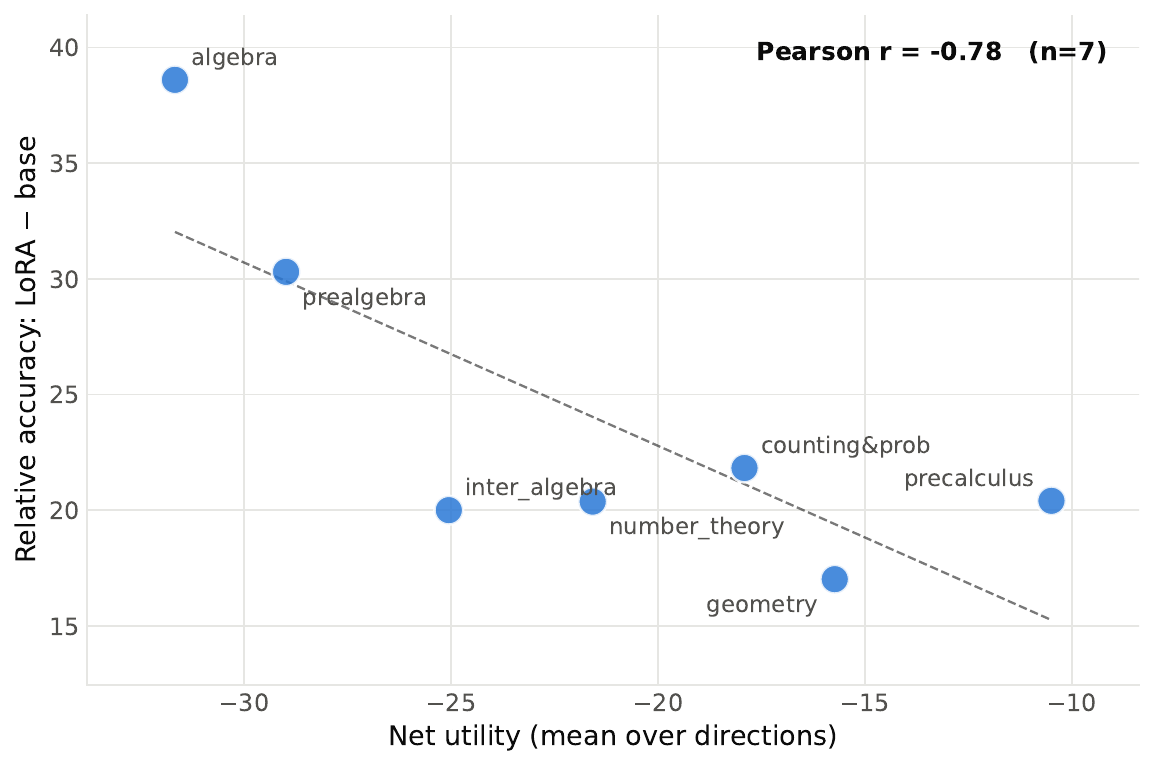}
    \caption{Net-utility vs performance on Hendrycks Math dataset. We observe that there is a negative correlation to the relative task performance and net utility}
    \label{fig:analysis}
    \vspace{-5.5mm}
\end{figure}

\section{Conclusion}
In this work, we show that uniform rank allocation across tasks and layers is not optimal and we introduce a novel metric called \textbf{Net Utility} metrics to solve this budget allocation greedily. For each task LoRA's singular direction the metric scores the singular direction's task benefit against the cross task interference. Using this metric we globally pool the singular directions across layers and tasks, and assign directions which have positive net utility. The chosen directions are then merged to form the merged LoRA. We apply this method over five merging methods across three merging spaces. The merging is done over two sets, one of seven vision tasks and another over six language tasks. That on average across all merging methods and spaces, vision tasks see an average improvement of +2.1\% and language tasks see a +2.2\% improvement. We also show interesting insight that net utility is correlated with task difficulty and could be interpreted as such.

\section{Limitations}
Our work is not without its limitations. At first, we restrict the net utility derivation to a linear merging, however one can derive the net utility to a specific merging algorithm that could further boost its performance. Secondly, we do not tune the hyperparameter $\lambda$, however, if given access to a validation set, tuning $\lambda$ would improve the performance. 

% {\let\thefootnote\relax\footnote{GenAI tools were used for help with writing and polishing the grammar of the paper.}}

\clearpage

\clearpage

% \section*{Acknowledgments}

% GenAI tools were used for only polishing the grammar of the paper.

% \bigskip
% \noindent Thank you for reading these instructions carefully. We look forward to receiving your electronic files!

\bibliography{aaai2027}

\clearpage

\clearpage
\appendix
% Preamble requirement: \newtheorem{proposition}{Proposition}

% \section{Technical Appendix}
% \label{sec:technical-appendix}

\setcounter{equation}{0}
\renewcommand{\theequation}{A.\arabic{equation}}

\section{Appendix}

\subsection{A.1 Layer-Wise Preservation Surrogate}
\label{app:layer-wise-preservation}

Let $J^l(x)$ be the output Jacobian with respect to layer $l$'s pre-activation under the shared reference model. A first-order expansion gives
\begin{equation}
  f_{\Delta W_i}(x)-f_{\Delta W_{\mathrm m}}(x)
  \approx
  \sum_{l=1}^{L}
  J^l(x)
  \bigl(\Delta W_i^l-\Delta W_{\mathrm m}^l\bigr)
  h^l(x).
  \label{eq:appendix-taylor-expansion}
\end{equation}

The sum contains cross-layer interactions. Cauchy--Schwarz and submultiplicativity give
\begin{equation}
  \begin{aligned}
    &\left\lVert
      \sum_{l=1}^{L}
      J^l(x)
      \bigl(\Delta W_i^l-\Delta W_{\mathrm m}^l\bigr)h^l(x)
    \right\rVert_2^2\\
    &\quad\le
    L\sum_{l=1}^{L}
    \left\lVert J^l(x)\right\rVert_{\mathrm{op}}^2
    \left\lVert
      \bigl(\Delta W_i^l-\Delta W_{\mathrm m}^l\bigr)h^l(x)
    \right\rVert_2^2.
  \end{aligned}
  \label{eq:appendix-cross-layer-bound}
\end{equation}

Summing over tasks and taking expectations preserves the bound. If the Jacobian norms are approximately constant over the adapted layers and inputs of interest, their common scale does not affect the minimizer, motivating the surrogate in Equation~\eqref{eq:preservation-surrogate}.

\subsection{A.2 Separable Upper Bound}
\label{app:separable-upper-bound}

\begin{proposition}[Separable upper bound]
  \label{prop:separable-upper-bound}
  If either $G_i$ or $\widehat G_i$ is used consistently and the corresponding normalization energy is nonzero, Equation~\eqref{eq:separable-upper-bound} holds for every $\lambda>0$.
\end{proposition}

\begin{proof}
  Write $G_i$ for the chosen geometry. For $M\ge2$, Young's inequality in the $G_i$-weighted seminorm gives, for any $a>0$,
  \begin{equation}
    \begin{aligned}
      &\left\lVert
        \bigl(\Delta W_i-K_i\bigr)-\sum_{j\ne i}K_j
      \right\rVert_{G_i}^2\\
      &\quad\le
      (1+a)\left\lVert\Delta W_i-K_i\right\rVert_{G_i}^2\\
      &\qquad+
      (1+a^{-1})
      \left\lVert\sum_{j\ne i}K_j\right\rVert_{G_i}^2.
    \end{aligned}
    \label{eq:appendix-young-bound}
  \end{equation}
  Cauchy--Schwarz gives
  \[
    \left\lVert\sum_{j\ne i}K_j\right\rVert_{G_i}^2
    \le
    (M-1)\sum_{j\ne i}\left\lVert K_j\right\rVert_{G_i}^2.
  \]
  Setting $a=(M-1)/\lambda$ and normalizing by $\lVert\Delta W_i\rVert_{G_i}^2$ gives Equation~\eqref{eq:separable-upper-bound}. For $M=1$, Equation~\eqref{eq:separable-upper-bound} holds with equality.
\end{proof}

\subsection{A.3 Net-Utility Decomposition}
\label{app:net-utility-decomposition}

\begin{proposition}[Net-utility decomposition]
  \label{prop:net-utility-decomposition}
  Assume $w_i>0$ for every indexed task; zero-energy updates contain no candidates and are omitted. Under the Gram geometry in Equation~\eqref{eq:gram-geometry}, each selected component has the fixed utility $g_{j,k}$ defined in Equation~\eqref{eq:net-utility}. Moreover, the summed normalized source-task residual and cross-task interference terms inside the right-hand side of Equation~\eqref{eq:separable-upper-bound} equal $M-\sum_{j,k}z_{j,k}g_{j,k}$.
\end{proposition}

\begin{proof}
  Define the $G$-weighted bilinear form as $\langle X,Y\rangle_G:=\operatorname{tr}(XGY^\top)$. Components from the same source task are orthogonal under this form:
  \begin{equation}
    \begin{aligned}
      \left\langle D_{j,k},D_{j,m}\right\rangle_{G_i}
      &=
      \sigma_{j,k}\sigma_{j,m}
      \left(u_{j,k}^{\top}u_{j,m}\right)\\
      &\quad{}\cdot
      \left(v_{j,k}^{\top}G_iv_{j,m}\right)
      =0,
      \qquad k\ne m.
    \end{aligned}
    \label{eq:appendix-component-orthogonality}
  \end{equation}
  Consequently, their weighted energies add. Equation~\eqref{eq:gram-geometry} gives
  \begin{equation}
    \begin{aligned}
      \left\lVert D_{i,k}\right\rVert_{G_i}^2
      &=\kappa_i\sigma_{i,k}^4,\\
      \left\lVert D_{j,k}\right\rVert_{G_i}^2
      &=
      \kappa_i\sigma_{j,k}^2
      \sum_n\sigma_{i,n}^2
      \left(v_{j,k}^{\top}v_{i,n}\right)^2,
      \qquad j\ne i,\\
      \left\lVert\Delta W_i\right\rVert_{G_i}^2
      &=\kappa_i\sum_n\sigma_{i,n}^4
      =\kappa_iw_i.
    \end{aligned}
    \label{eq:appendix-gram-energies}
  \end{equation}
  The normalization therefore cancels $\kappa_i$. Summing the normalized source-task residual and cross-task interference terms in Equation~\eqref{eq:separable-upper-bound} gives
  \begin{equation}
    \begin{aligned}
      &\sum_{i=1}^{M}
      \frac{
        \left\lVert\Delta W_i-K_i\right\rVert_{G_i}^2
        +
        \lambda\sum_{j\ne i}\left\lVert K_j\right\rVert_{G_i}^2
      }{
        \left\lVert\Delta W_i\right\rVert_{G_i}^2
      }\\
      &=
      \sum_{i=1}^{M}\sum_{k=1}^{r_i}
      (1-z_{i,k})
      \frac{\sigma_{i,k}^4}
      {\sum_n\sigma_{i,n}^4}\\
      &\quad+
      \lambda\sum_{j=1}^{M}\sum_{k=1}^{r_j}
      z_{j,k}\sigma_{j,k}^2
      \sum_{i\ne j}\sum_{n=1}^{r_i}
      \frac{\sigma_{i,n}^2}
      {\sum_m\sigma_{i,m}^4}
      \left(v_{j,k}^{\top}v_{i,n}\right)^2\\
      &=
      M-\sum_{j=1}^{M}\sum_{k=1}^{r_j}z_{j,k}g_{j,k}.
    \end{aligned}
    \label{eq:appendix-net-utility-decomposition}
  \end{equation}
  This proves the claimed decomposition.
\end{proof}

\subsection{A.4 Heuristic Row-Space Utility}
\label{app:heuristic-row-space-utility}

The utility for $\widehat G_i=\kappa_iV_iV_i^\top$ follows from the same bilinear-form orthogonality argument as Proposition~\ref{prop:net-utility-decomposition}. In particular,
\begin{equation}
  \begin{aligned}
    \left\lVert D_{i,k}\right\rVert_{\widehat G_i}^2
    &=\kappa_i\sigma_{i,k}^2,\\
    \left\lVert D_{j,k}\right\rVert_{\widehat G_i}^2
    &=
    \kappa_i\sigma_{j,k}^2
    \sum_n
    \left(v_{j,k}^{\top}v_{i,n}\right)^2,
    \qquad j\ne i,\\
    \left\lVert\Delta W_i\right\rVert_{\widehat G_i}^2
    &=\kappa_i\sum_n\sigma_{i,n}^2
    =\kappa_i\widehat w_i.
  \end{aligned}
  \label{eq:appendix-row-space-energies}
\end{equation}

Thus, define the heuristic task scale and utility as
\begin{equation}
  \begin{aligned}
    \widehat w_i
    &:=
    \sum_{n=1}^{r_i}\sigma_{i,n}^2,\\
    \widehat g_{j,k}
    &:=
    \underbrace{\frac{\sigma_{j,k}^2}{\widehat w_j}}_{\text{source-task benefit $\tilde{\pi}$}}
    -
    \lambda\sigma_{j,k}^2
     \underbrace{\sum_{i\ne j}\sum_{n=1}^{r_i}
    \frac{
      \left(v_{j,k}^{\top}v_{i,n}\right)^2
    }{
      \widehat w_i
    }}_{\text {cross-task interference $\tilde{C}$ }}.
  \end{aligned}
  \label{eq:appendix-row-space-utility}
\end{equation}

Substituting Equation~\eqref{eq:appendix-row-space-energies} into the summed normalized terms inside the right-hand side of Equation~\eqref{eq:separable-upper-bound} cancels $\kappa_i$ and gives
\begin{equation}
  M-\sum_{j=1}^{M}\sum_{k=1}^{r_j}z_{j,k}\widehat g_{j,k}.
  \label{eq:appendix-row-space-decomposition}
\end{equation}

The heuristic variant replaces $w_i^l$ and $g_{j,k}^l$ in Algorithm~\ref{alg:net-utility} with $\widehat w_i^l$ and $\widehat g_{j,k}^l$, respectively; all other steps are unchanged.

\begin{table}[t]
\centering
\caption{Hyperparameters for LoRA adapter training on Qwen3-4B.}
\label{tab:lora-hyperparams}
\begin{tabular}{ll}
\toprule
Hyperparameter & Value \\
\midrule
Base model & Qwen3-4B-Instruct \\
Task type & Sequence classification (3 classes) \\
LoRA rank $r$ & 16 \\
LoRA scaling $\alpha$ & 16 \\
LoRA dropout & 0.1 \\
Target modules & $W_q$, $W_k$, $W_v$, $W_o$ \\
Optimizer & AdamW \\
Learning rate & $3\times10^{-5}$ \\
LR schedule & Linear, 6\% warmup \\
Batch size & 1 \\
Max.\ epochs & 10 \\
Max.\ sequence length & 2{,}000 \\
Evaluation interval & 4{,}000 steps \\
Early-stopping patience & 3 evaluations \\
\bottomrule
\end{tabular}
\end{table}

\subsection{A.5 Computational Complexity}
\label{app:computational-complexity}

For a rank-$r$ LoRA update $\Delta W_i=B_iA_i$, its thin SVD can be computed from QR factorizations of $B_i$ and $A_i^\top$, followed by an $r\times r$ SVD. This costs $O(r^2(d_{\mathrm{in}}+d_{\mathrm{out}}))$ per adapter without materializing the dense update, where $d_{\mathrm{in}}$ and $d_{\mathrm{out}}$ are the layer input and output widths.

At each layer, forming all $V_j^\top V_i$ costs $O(M^2r^2d_{\mathrm{in}})$ time and $O(M^2r^2)$ memory. If $C$ components are pooled, selecting them costs $O(C\log C)$

\subsection{B LoRA Training Details for QWEN3-4B adapters}

We train one LoRA adapter per NLI task on Qwen3-4B-Instruct as the base
model, using the sequence-classification formulation with a three-way
classification head. Adapters are trained with the PEFT
library; the base model weights remain frozen throughout,
and only the LoRA matrices and the classification head are updated.

\paragraph{LoRA configuration.}
LoRA modules are attached to all four attention projection matrices
($W_q$, $W_k$, $W_v$, $W_o$) in every transformer layer, with rank
$r=16$, scaling factor $\alpha=16$ (i.e., $\alpha/r=1$), and dropout of
$0.1$. The feed-forward (MLP) layers are left unmodified.

\paragraph{Tasks and label space.}
We train adapters on six NLI datasets: SNLI, MNLI, SICK, QNLI, RTE, and
SciTail. All tasks share a unified three-class output space
(0: \emph{entailment}, 1: \emph{neutral}, 2: \emph{contradiction}) so
that all adapters and classification heads remain shape-compatible for
merging. Binary datasets are mapped into this space: for RTE and QNLI,
\emph{not-entailment} is assigned to the contradiction class, and for
SciTail the \emph{neutral} label retains class~1. For each binary task,
the unused class is masked by setting its logit to $-10^{10}$ before
computing the loss and during evaluation.

% \paragraph{Tokenization.}
% Premise--hypothesis pairs are tokenized jointly with the Qwen3
% tokenizer, truncated to a maximum length of 2{,}000 tokens, and padded
% dynamically to the longest sequence in each batch.

\paragraph{Optimization.}
All adapters are trained with AdamW at a learning rate of
$3\times10^{-5}$ under a cross-entropy objective, with a linear
learning-rate schedule and 6\% warmup, for at most 10 epochs.
Validation accuracy is evaluated every 4{,}000 steps; we retain the
checkpoint with the highest validation accuracy and apply early
stopping after three consecutive evaluations without improvement. The
best checkpoint is then evaluated once on the held-out test set. The
full hyperparameter configuration is summarized in
Table~\ref{tab:lora-hyperparams}.

\end{document}